\documentclass[conference]{ieeeconf}

\IEEEoverridecommandlockouts
\usepackage{cite}
\usepackage{amsmath,amssymb,amsfonts}
\usepackage{algorithmic}
\usepackage{graphicx}
\usepackage{textcomp}
\usepackage{xcolor}
\usepackage{float}
\usepackage[caption = false]{subfig}
\usepackage[bookmarks=true]{hyperref}

\def\BibTeX{{\rm B\kern-.05em{\sc i\kern-.025em b}\kern-.08em
    T\kern-.1667em\lower.7ex\hbox{E}\kern-.125emX}}
    
\newtheorem{lemma}{Lemma}
\newtheorem{theorem}{Theorem}
\newtheorem{definition}{Definition}
    
\newcommand{\vpnote}[1]%
    {\textcolor{cyan}{\textbf{VP: #1}}}
\newcommand{\asnote}[1]%
    {\textcolor{orange}{\textbf{AS: #1}}}

\begin{document}

\title{Exponentially Stable First Order Control on Matrix Lie Groups\\
}

\author{Valmik Prabhu*, Amay Saxena*, and S. Shankar Sastry%
\thanks{
Valmik Prabhu is with the department of Mechanical Engineering, and Amay Saxena and S. Shankar Sastry are with the department of Electrical Engineering and Computer Science at UC Berkeley. \newline
Corresponding Author: \href{mailto:valmik@berkeley.edu}{\tt \small{valmik@berkeley.edu}}
\newline
$^*$ indicates equal contribution.
}
}

\maketitle

\begin{abstract}

We present a novel first order controller for systems evolving on matrix Lie groups, a major use case of which is Cartesian velocity control on robot manipulators. This controller achieves global exponential trajectory tracking on a number of commonly used Lie groups including the Special Orthogonal Group SO(n), the Special Euclidean Group SE(n), and the General Linear Group over complex numbers GL(n, C). Additionally, this controller achieves local exponential trajectory tracking on all matrix Lie groups. We demonstrate the effectiveness of this controller in simulation on a number of different Lie groups as well as on hardware with a 7-DOF Sawyer robot arm.

\end{abstract}


\section{Introduction}



A number of robot control tasks such as welding, painting, and part alignment benefit when the robot's trajectory is defined in terms of its end effector pose rather than its joint angles. This is called \textit{Cartesian control}. The robot's end effector is able to translate and rotate in 3D space, and thus it's pose evolves on $SE(3)$, the set of rigid body transformations in three dimensions. $SE(3)$ is a \textit{Lie group}, a smooth manifold that possesses additional algebraic and geometric structure. 

While there has been much prior work in this area, most Cartesian controllers use local coordinate parameterizations of the end effector rotation, such as Euler angles. Local parameterizations exhibit singularities such as gimbal lock, and therefore these coordinates are only valid for a local set of end effector configurations. To our knowledge, no one has yet proven a globally stable Cartesian velocity controller.

In this paper, we present a first order tracking controller for fully actuated systems that evolve on matrix Lie groups. This controller exhibits global exponential tracking on a number of Lie groups, including $SE(3)$ and $SO(3)$, which allows us to define globally exponentially stable Cartesian controllers for robots. In addition, our control law also provides local exponential stability on all matrix Lie groups.

In Section \ref{sec:related_work} we discuss prior work in Cartesian control on robots and geometric control on Lie groups, and in Section \ref{sec:math_background} we review the mathematics of matrix Lie groups and present the notation we use. In Section \ref{sec:formulation} we formulate the first order tracking problem on matrix Lie groups, before justifying our control law in Sections \ref{sec:local_proof} and \ref{sec:global_proof}. We first present a proof of local exponential tracking with a small radius of convergence. This proof requires substantially less mathematical machinery than the proof for globally exponential stability, and we believe that the local proof provides some insight useful in understanding the more general case in Section \ref{sec:global_proof}. We also provide a discrete time version of the proof in Section \ref{sec:discrete_proof}. Finally, in Section \ref{sec:results} we demonstrate our controller's performance experimentally in simulation as well as on a Sawyer robot arm.

\begin{figure}[t!]
    \centering
    \includegraphics[ width=\columnwidth]{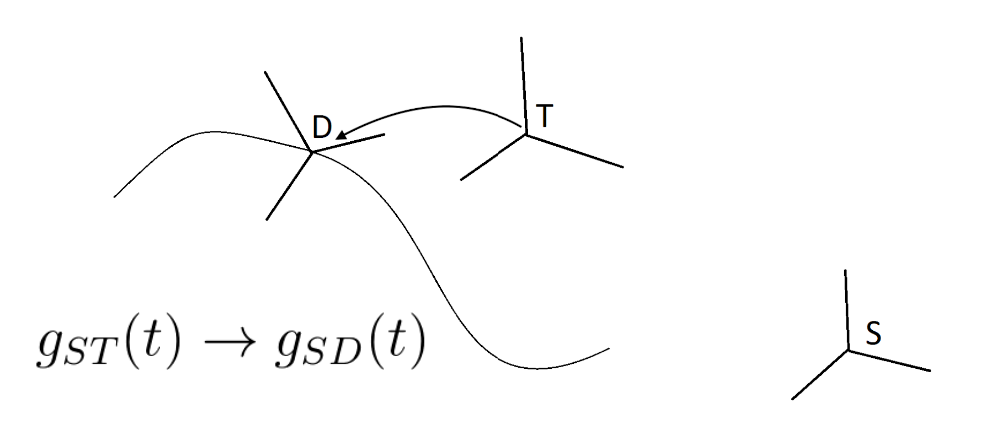}
    \caption{Visualization of the trajectory tracking problem where the group $G$ on which the system evolves is the Special Euclidean Group $SE(3)$. The system's pose is denoted by the coordinate frame $T$, which is represented with respect to an arbitrary reference frame $S$ by the transform $g_{ST}(t)$. Its desired pose is denoted by the coordinate frame $D$ and is represented by the transform $g_{SD}(t)$. The control input is the system body velocity $u = g_{ST}^{-1} \dot{g}_{ST}$. The goal of the tracking problem is to design a controller for $u$ that makes $g_{ST}(t)$ exponentially converge to $g_{SD}(t)$}
    \label{fig:tracking_problem}
    \vspace{-0.5cm}
\end{figure}

\section{Related Work} \label{sec:related_work}

Initial work in first order Cartesian control was done by Whitney \cite{WhitneyJacobian}, who used a pseudoinverse-Jacobian based controller and local coordinates. Khatib later developed a second order Cartesian controller for manipulators, by defining Cartesian analogs of the inertia, Coriolis, and gravity terms used in computed torque control \cite{khatib1987unified}. While these Cartesian control techniques have made their way into all the prominent robotics textbooks \cite{SpringerHandbook, CanudasDeWit, siciliano2010book, MLS, JJCraig, SpongBook}, only one book provides a control law that does not use local coordinates. \textit{Modern Robotics}, by Lynch and Park, presents a Cartesian control law very similar to our own \cite{ModernRobotics}. However, they present it only as an analog to their joint-space controller, and do so without proof. We contacted the authors, and they informed us that they do not have a proof.

There has also been substantial interest in developing geometric controllers for systems on Lie groups, beginning with Brockett \cite{brockett1973lie}. Lewis and Bullo did some of the initial work on second order control for fully actuated systems on Riemannian manifolds (which include Lie groups) \cite{BulloMurray, BulloLewisTextbook}. By defining the kinetic energy as a Riemannian metric on the manifold, they were able to derive the Lagrangian dynamics for these systems and define controllers with almost-global asymptotic tracking and local exponential convergence. Here they define ``almost-global" asymptotic stability as asymptotic convergence from all initial conditions apart from a nowhere dense measure zero set of states. In $SO(3)$ this measure zero set is the set of all rotations of exactly $180^\circ$. Intuitively, this occurs because there are two paths back to the identity with exactly equal length and opposite control input. As a result, the controller cannot pick a direction. In practice this is not an issue; since this set of states is measure zero, it will be encountered with probability zero in the discrete-time setting. 



Maithirpala extended Lewis and Bullo's work by defining a \textit{configuration error} on the Lie group, and using Lagrangian methods to stabilize that configuration error about the identity \cite{maithripala}. We define configuration error in Section \ref{sec:math_background} in a similar manner. Taeyoung Lee also extended Lewis and Bullo's work by designing a hierarchical controller for UAVs that allows for almost-globally asymptotic tracking of four degrees of freedom (translational position and yaw), despite underactuated vehicle dynamics \cite{McClamrochDrone}.

Robot control is not the only field with dynamic systems on Lie groups. Geometric control is also used in state estimation for SLAM \cite{sola2017quaternion} as well as in error correction in quantum computing \cite{zhangQbit}. It's likely our results could be useful in these fields as well.





\section{Mathematical Background} \label{sec:math_background}

Here we present an overview of the concepts from matrix Lie theory that we use in this paper. We follow the notation in \cite{MLS}. For a more detailed coverage of matrix Lie groups, we direct the reader to \cite{hall2015lie}. Most of this section is drawn from these two references.

\subsection{Matrix Lie Groups}
Matrix Lie groups are continuous groups represented by square matrices. As a result, they possess a number of useful algebraic and geometric properties. We denote a Lie group by $G$, and a member of the group by $g \in G$.

\vspace{2mm}
\subsubsection{Definition of a Matrix Group}
A Matrix group is a set $G$ of $n \times n$ matrices equipped with the standard matrix multiplication that satisfy the four \textit{group axioms}. It is closed under matrix multiplication (i.e. $\forall g_1, g_2 \in G, g_1 \cdot g_2 \in G$), contains the identity matrix, and is closed under inversion (i.e. $g \in G \implies g^{-1} \in G$). The final group axiom is that the multiplication operation should be associative, which we know to be the case for standard matrix multiplication.

\vspace{2mm}
\subsubsection{Algebraic Properties of a Matrix Group} 
The elements of Matrix group are all $n \times n$ invertible matrices, and thus possess all the properties of linear transformations. In particular, since $g$ is invertible, the columns of $g$ form a basis for $\mathbb{R}^n$ (or $\mathbb{C}^n$). Thus we can represent every $g \in G$ as a linear transformation between two coordinate frames $g_{AB} : B \rightarrow A$. $g_{AB}$ maps vectors in frame $B$ to vectors in frame $A$. We have $g_{AB}^{-1} = g_{BA}$ and $g_{AB} g_{BC} = g_{AC}$. 

\vspace{2mm}
\subsubsection{Geometric Properties of Matrix Lie Groups}
Matrix Lie groups are continuous groups. This means that their elements lie on a smooth manifold in $\mathbb{R}^{n \times n}$ (or $\mathbb{C}^{n \times n}$). In particular, this property means that there exists a derivative $\dot{g}, \forall g \in G$. This derivative lies in the tangent space of the manifold at $g$. The tangent space at the identity element $I \in G$ is called the Lie algebra $\mathfrak{g}$.

A useful property of Lie groups is that the tangent space $\dot{g}$ is diffeomorphic to the group's Lie algebra $\mathfrak{g}$. By multiplying $\dot{g}$ by $g^{-1}$, we can ``rotate" $\dot{g}$ to lie in $\mathfrak{g}$. There are two ways to do this. By left-multiplying we get
\begin{equation}
    g_{AB}^{-1} \dot{g}_{AB} = g_{BA} \dot{g}_{AB} = \hat{V}^b_{AB} \in \mathfrak{g}_B
\end{equation}
This is called the \textit{body velocity} and it's defined in the second, or ``body" frame $B$. By right-multiplying we get
\begin{equation} \label{eq:spatial_velocity}
    \dot{g}_{AB} g_{AB}^{-1} = \dot{g}_{AB} g_{BA} = \hat{V}^s_{AB} \in \mathfrak{g}_A
\end{equation}
This is called the \textit{spatial velocity} and it's defined in the first, or ``spatial" frame $A$.

\subsection{Some Common Matrix Lie Groups}

Here we briefly list some common matrix Lie groups:
\begin{description}
\item[$GL(n, \mathbb{C})$] \quad \,\, The general linear group on $\mathbb{C}$. The set of all invertible matrices in $\mathbb{C}^{n \times n}$. All matrix Lie groups are subgroups of this group. Its Lie algebra, $\mathfrak{gl}(n, \mathbb{C}$ is the set of complex $n \times n$ matrices.
\item[$SO(n)$] \quad \,\, The special orthogonal group, also known as the ``rotation group". The set of all matrices in $\mathbb{R}^{n \times n}$ such that $g^T = g^{-1}$ and $\det(g) = 1$. $SO(2)$ is the set of planar rotations, and $SO(3)$ is the set of 3D rotations about an axis. The Lie algebra of $SO(n)$, $\mathfrak{so}(n)$, is the set of skew symmetric matrices in $\mathbb{R}^{n \times n}$.
\item[$SU(n)$] \quad \,\, The special unitary group, or the set of all matrices in $\mathbb{C}^{n \times n}$ where $g^* = g^{-1}$ and $det(g) = 1$. This is the complex analog to $SO(n)$, and its Lie algebra, $\mathfrak{su}(n)$ is the set of skew-Hermitian matrices in $\mathbb{C}^{n \times n}$.
\item[$SE(n)$] \quad \,\, The special Euclidian group, or the set of all rigid body transforms, and is the Cartesian product $SE(n) = SO(n) \times \mathbb{R}^n$. 
\end{description}

\subsection{The exponential map}

By manipulating \ref{eq:spatial_velocity} we get the following matrix differential equation.
\begin{equation}
    \dot{g}(t) = \hat{V}^s(t) g(t)
\end{equation}
If we hold the spatial velocity constant, the solution to this differential equation is
\begin{equation}
    g(T) = e^{\hat{V}^s T} g(0)
    \label{eq:exp_map}
\end{equation}
Here, $e^{\hat V^s}$ is the matrix exponential of $\hat V^s$, defined using the usual power series. Since $g(T)$ and $g(0)$ are both elements of $G$, by closure, so is $e^{\hat V^s}$. The exponential map takes elements from the Lie algebra and maps them to elements of the Lie group. 

In light of \ref{eq:exp_map}, we see that the exponential of a velocity $\hat V$ in the Lie algebra is the result of starting off at the identity, and then wrapping around the manifold along the ``great circle" or \textit{geodesic} in the direction of $\hat V$. Geometrically, this corresponds to starting off at the identity, and then moving along the manifold with constant velocity $\hat V$ for 1 second. The result of this motion is the resulting element of the Lie group and the output of the exponential.

In general the exponential map is not surjective for an arbitrary matrix Lie group. For a given group $G$, there may exist some $g \in G$ such that $\nexists \hat{\xi} \in \mathfrak{g}, \, st. \, e^{\hat{\xi}} = g$. A number of common matrix Lie groups, including $SO(n)$, $SE(n)$, $SU(n)$, and $GL(n, \mathbb{C})$, possess globally surjective exponential maps , but others, such as $GL(n, \mathbb{R})$ do not. However, there is always a region near the identity in which the map is locally surjective. This region is at least the region where $g = e^{\hat{\xi}}, \, |\hat{\xi}| < \ln 2$, but it may be larger. The minimal region often corresponds to the region in which the spectral radius $|g - I| < 1$, but it does not always do so.

Within this surjective region, we can define a parameterization for $G$ called \textit{exponential coordinates}. The exponential coordinates for a element $g \in G$ are $\hat{\xi} \in \mathfrak{g}$, where $g = e^{\hat{\xi}}$.

\subsection{The logarithmic map}
The logarithm inverts the exponential map, where that is possible. The logarithm of an element $g \in G$ is the solution $\hat\xi \in \mathfrak{g}$ of the equation $g = e^{\hat\xi}$. This equation may in general have no solutions or many solutions. We define the logarithm only on the image $\exp(\mathfrak{g})$. Even so, the logarithm may be a multiple valued function. There may be many ways to recover a single valued function from this logarithm.

Locally, the exponential map is a diffeomorphism from an open neighbourhood of $0 \in \mathfrak{g}$ to an open neighbourhood of $I \in G$ \cite{hall2015lie}. In this region then, the logarithm is well defined and smooth.

For matrix lie groups, we may also consider the \textit{matrix logarithm}, which is the solution to $e^Z = X$ for a given matrix $X$. The matrix logarithm can be made to be unique (in fact, holomorphic), by picking a branch. Here, we consider chiefly the principal branch of the matrix logarithm. In particular, we can restrict the domain of the logarithm to only those matrices whose spectrum avoids the non-positive real axis. For any such matrix $X$ from a matrix Lie group, there is a unique matrix logarithm $Z$ all of whose eigenvalues lie in the strip $-\pi < \operatorname{Im} z < \pi$. This matrix is the \textit{principal logarithm} of $X$. Close to the identity, the principal logarithm can be written as a convergent power series
\begin{equation}
    \log(X) = \sum_{k=1}^{\infty} (-1)^{k+1} \frac{(X - I)^{k}}{k}
    \label{eq:log_series}
\end{equation}
which converges whenever $||X - I|| < 1$, though it may converge elsewhere as well (for instance, if $X - I$ is nilpotent).

We will henceforth use the operator $\log$ to denote such a well defined logarithmic map from a subset of the Lie group to its corresponding Lie algebra. Often, we will use the principal branch of the matrix logarithm for this, wherever it is defined and lies in the Lie algebra. In some groups such as $SE(3)$ and $SO(3)$, there exist analytic formulas for the logarithmic map, which we use.

\subsection{SO(3) as a Matrix Lie Group}

The set of rotation matrices forms a 3 dimensional matrix Lie group called the \textit{special orthogonal group}, notated $SO(3) = \{R \in \mathbb R^{3\times 3} : R^T R = R R^T = I, \ \det (R) = 1\}$. The associated lie algebra is $\mathfrak{so}(3)$. $\mathfrak{so}(3)$ is a 3 dimensional vector space, and is in fact exactly the set of $3 \times 3$ skew-symmetric matrices $\mathfrak{so}(3) = \{X \in \mathbb R^{3 \times 3}: X^T = -X\}$.

We further define the "hat" and "vee" operators. The hat operator $(\cdot)^\wedge: \mathbb R^3 \rightarrow \mathfrak{se}(3)$ is the canonical isomorphism between $\mathbb R^3$ and $\mathfrak{so}(3)$ and the vee operator $(\cdot)^{\vee}$ is its inverse. Given $\omega = (\omega_x, \omega_y, \omega_z)^T \in \mathbb R^3$ we write $\hat \omega = \omega^\wedge$ as
\begin{align}
    \hat \omega = \left[\begin{array}{ccc}
    0 & -\omega_{z} & \omega_{y} \\
    \omega_{z} & 0 & -\omega_{x} \\
    -\omega_{y} & \omega_{x} & 0
    \end{array}\right]
\end{align}
The exponential map on $\mathfrak{so}(3)$ into $SO(3)$ is surjective. Given any rotation matrix $R \in SO(3)$, there exists $\omega \in \mathbb R^3$ such that $R = e^{\hat\omega}$. The matrix $\hat\omega$ is called the ``exponential coordinate" of the rotation $R$. 

$\hat\omega = \omega^\wedge$ has a relevant geometric interpretation. The statement that $R$ can be written as $e^{\hat \omega}$ amounts to saying that the rotation described by $R$ can be implemented by a single rotation of $||\omega||_2$ radians about the unit axis in the direction of $\omega$. The surjectivity of the exponential map, then, is the statement that \textit{any} rotation matrix can be realized as a single rotation about some axis.

\subsection{SE(3) as a Matrix Lie Group}

Consider a rigid body moving through free space. A standard way to track the motion of a rigid body is by fixing a coordinate frame $T$ rigidly on the body and then describing the motion of this coordinate frame by tracking its orientation and position relative to a fixed world reference frame $S$, using a rotation matrix $R_{ST} \in SO(3)$ and a translation vector $p_{ST} \in \mathbb R^3$. Using homogeneous coordinates, we can express both the rotational and translational components together by stacking them into a single $4\times 4$ transformation matrix $g_{ST}$.
\begin{align}
    g_{ST} = \begin{bmatrix}
    R_{ST} & p_{ST} \\
    0 & 1
    \end{bmatrix}
    \label{eq:rigid_transforms}
\end{align}
The set of matrices of the form of \ref{eq:rigid_transforms} forms a 6 dimensional matrix Lie group called the \textit{special euclidean group} in 3 dimensions, notated $SE(3)$. Its associated Lie algebra is $\mathfrak{se}(3)$, the set of all rigid body velocities.

$\mathfrak{se}(3)$ is a 6 dimensional vector space. An element $\hat \xi \in \mathfrak{se}(3)$ takes the form
\begin{align}
    \hat \xi = \left[\begin{array}{cc}
    \hat \omega & v\\
    0 & 0 \\
    \end{array}\right]
\end{align}
where $\hat\omega \in \mathfrak{so}(3)$ and $v \in \mathbb R^3$. As in the case of $SO(3)$, we define "hat" and "vee" operators between $\mathfrak{se(3)}$ and $\mathbb R^6$. For the element $\hat \xi \in \mathfrak{se}(3)$, we get $\xi = \hat \xi ^{\vee}$ by stacking $v$ and $\omega = \hat\omega^\vee$ so that $\hat\xi^\vee = (v, \omega) \in \mathbb R^6$.

The exponential map from $\mathfrak{se}(3) \rightarrow SE(3)$ is also surjective. For any rigid body transform $g \in SE(3)$, we can find a $\hat\xi \in \mathfrak{se}(3)$ such that $g = e^{\hat\xi}$. In fact, if $\xi = (\omega, v)$, then the rotational matrix component $R$ of $g$ is exactly $e^{\hat \omega}$. So standard methods of finding $\xi$ start off by using the logarithm on $SO(3)$ to find $\omega$, and then solving for $v$ algebraically. We likewise call the matrix $\hat\xi$ the ``exponential coordinate" of $g$.

Elements of $\mathfrak{se}(3)$ have a geometric interpretation. If $g(t)$ is a smooth path in $SE(3)$, it describes the smooth motion of some rigid body through space. We say that the spatial velocity of the rigid body at time $t$ is $\hat V^s$ if it satisfies $\dot g = \hat V^s g$ at $t$. This velocity exists for any such $g(t)$ since $\dot g g^{-1} \in \mathfrak{se}(3)$. When this is the case, the velocity of any point $p$ on the rigid body as measured in homogeneous coordinates in the inertial frame, is given by $\dot p = \hat V^s p$. A similar interpretation exists for the body velocity $\hat V^b = g^{-1}\dot g$. This time, the velocity and position of the point $p$ is measured in the instantaneous body reference frame. Then, to express $g \in SE(3)$ as the exponential of a velocity $\hat \xi \in \mathfrak{se}(3)$, is to say that there exists a smooth path $\tilde g(t)$ with $\tilde g(0) = I$, $\tilde g(1) = g$, that evolves with uniform spatial velocity $\hat \xi$. That is to say, $\hat\xi$ is that spatial velocity that if a rigid body executes, starting from the identity, will bring it to the configuration $g$ in 1 second.


There exist analytic closed form expressions for the logarithm on $SE(3)$. We direct the reader to \cite{MLS} for an exposition of such expressions. The salient fact is that the logarithm can be made to be well defined on $SE(3)$ globally. In particular, it suffices to restrict the output of the logarithm to those elements $\xi = (v, \omega) \in \mathfrak{se}(3)$ satisfying $||\omega|| \leq \pi$. In this case, these formulas for the logarithm coincide with the principal matrix logarithm where the latter is defined. Hence, whenever $g$ has no eigenvalues on the negative real axis, the logarithm on $SE(3)$ is holomorphic.

Elsewhere, this logarithm is still well defined, but discontinuous. It's points of discontinuity are exactly the elements of $SE(3)$ that have eigenvalues on the negative real axis. For any rigid transform $g = e^{\hat\xi}$ with $\xi = (v, \omega)$, $g$ has eigenvalues $e^{i||\omega||}, e^{-i||\omega||}$ and 1, with the unit eigenvalue appearing twice. So, the points of discontinuity of the logarithm on $SE(3)$ are exactly those rigid transforms whose rotational components correspond to a rotation of $\pi$ radians about the $\omega$ axis. There are two possible rotation axes whenever $||\omega|| = \pi$ that differ by a sign, which would both serve well as the output of the logarithm at these points. We pick the output that occurs as the limit as $||\omega|| \rightarrow \pi^-$ radians about the same axis.


\section{First Order Control on Matrix Lie Groups}

\subsection{Problem Formulation} \label{sec:formulation}

We consider two time-varying transforms. $g_{ST}(t) \in G$ is the system state defined with respect to some arbitrary reference frame $S$. We, define a \textit{left-invariant} control system, which means that we control the body velocity of the system state.
\begin{equation}
    \dot{g}_{ST}(t) = g_{ST}(t) u(t), u \in \mathfrak{g}_T
\end{equation}
Equivalent results can be defined for right-invariant systems, where the control input is the spatial velocity, but we do not do so in this paper. $g_{SD}(t) \in G$ is the reference trajectory of the system, again defined with respect to the reference frame $S$. We assume that the reference trajectory is differentiable with a piece-wise continuous derivative.

We define the configuration error $g_{TD}(t) \in G$ as
\begin{equation}
    g_{TD}(t) = g_{ST}^{-1}(t) g_{SD}(t)
\end{equation}
If this configuration error is the identity matrix then the system state is tracking the trajectory. We'll then define the state error $\hat{\xi}_{TD}(t) \in \mathfrak{g}$ such that
\begin{gather}
    g_{TD}(t) = e^{\hat{\xi}_{TD}(t)} \\
    \hat{\xi}_{TD}(t) = \log(g_{TD}(t))
\end{gather}
Note that $\hat{\xi}_{TD}(t)$ is defined with respect to the $T$ frame, so $\hat{\xi}_{TD}(t) \in \mathfrak{g}_T$. Meanwhile, $\log(g_{ST}(t)) \in \mathfrak{g}_S$, which is a different vector space. Since $e^{\mathbf{0}} = I$, if the state error $\hat{\xi}_{TD}(t)$ is the zero matrix, then the configuration error $g_{TD}(t)$ must be the identity matrix. Thus, driving the state error to zero is equivalent to driving the configuration error to the identity. In the rest of the paper, we will drop the time dependence from the notation where convenient.

\subsection{Proof of Local Exponential Convergence} \label{sec:local_proof}

First we present a local stability proof. While this paper's main contribution is the global proof, this local proof allows us to define some useful tools and intuition without the added mathematical machinery of the global proof. Our goal is to find a control input $u$ such that the dynamics of the state error are
\begin{equation}
    \dot{\hat{\xi}}_{TD} = -k \hat{\xi}_{TD}
\end{equation}
where $k \in \mathbb{R}_+$ is a positive control gain. If this is the case, the state error will be 
\begin{equation}
    \hat{\xi}_{TD}(t) = e^{-k t} \hat{\xi}_{TD}(0)
\end{equation}
And will thus exponentially converge to the zero matrix, with rate of convergence $k$.

Our first task is to represent $\dot{\hat{\xi}}_{TD}$ as a function of $g_{TD}$ and its derivative. To do this, we use the power series expansion of the matrix logarithm as shown in \ref{eq:log_series}.
\begin{equation}
    \log(g_{TD}) = \sum_{k=1}^{\infty} (-1)^{k+1} \frac{(g_{TD} - I)^{k}}{k}
\end{equation}
This power series is a local representation of the matrix log, and is only guaranteed to converge when the spectral radius $|g_{TD} - I| < 1$. We now make the assumption that $g_{TD}$ commutes with its derivative. If this assumption holds, we can pull $\dot{g}_{TD}$ out of the summation, yielding
\begin{equation}
    \dot{\hat{\xi}}_{TD} 
        = \dot{g}_{TD} \sum_{k=0}^{\infty} (-1)^{k} (g_{TD} - I)^{k}
        = \dot{g}_{TD} \sum_{k=0}^{\infty} (I - g_{TD})^k
\end{equation}
This is a Neumann series \cite{higham}. Since $|g_{TD} - I| < 1$, the series converges and we have
\begin{equation}
    \sum_{k=0}^{\infty} (I - g_{TD})^k = g_{TD}^{-1}
\end{equation}
Thus we have
\begin{equation} \label{eq:dxi_local}
    \dot{\hat{\xi}}_{TD}
        = \dot{g}_{TD} g_{TD}^{-1} = g_{TD}^{-1} \dot{g}_{TD}
\end{equation}
Now we express $\dot{g}_{TD}$ in terms of our control input $u$. Since $g_{TD} = g_{ST}^{-1} g_{SD}$ the chain rule yields
\begin{align}
    \dot{g}_{TD} &= - g_{ST}^{-1} \dot{g}_{ST} g_{ST}^{-1} g_{SD} + g_{ST}^{-1} \dot{g}_{SD} \nonumber\\
    &= - u g_{ST}^{-1} g_{SD} + g_{ST}^{-1} \dot{g}_{SD} \label{eq:dg_TD}
\end{align}
Now we can design a stabilizing control law for $u$.
\begin{lemma}\label{lemma:local_stability}
A controller of the form
\begin{equation}
    u = k \hat{\xi}_{TD} + g_{TD} \hat{V}^b_{SD} g_{TD}^{-1}
\end{equation}
with $k \in \mathbb{R}_+$ a positive scalar control gain and $\hat{V}^b_{SD} = g^{-1}_{SD} \dot{g}_{SD}$, results in local exponential trajectory tracking in the subset of the region $|g_{TD} - I| < 1$ where the principal matrix logarithm lies in the Lie algebra $\mathfrak g$. Note that this region is always nonempty, connected, and contains the identity.

\end{lemma}

\begin{proof}
By plugging the controller into \eqref{eq:dg_TD}, we get
\begin{equation}
    \dot{g}_{TD} = - k \hat{\xi}_{TD} g_{TD}
\end{equation}
Since $g_{TD} = e^{\hat{\xi}_{TD}}$, $\hat{\xi}_{TD}$ commutes with $g_{TD}$. Thus, $\dot{g}_{TD}$ commutes with $g_{TD}$, fulfilling the assumption made in \eqref{eq:dxi_local}. We now have
\begin{equation}
    \dot{\hat{\xi}}_{TD} = \dot{g}_{TD} g_{TD}^{-1} = - k \hat{\xi}_{TD}
\end{equation}
Since \eqref{eq:dxi_local} is only valid for $|g_{TD} - I| < 1$, we have local exponential convergence in this region.
\end{proof}

\begin{figure*}[t]
\centering
\subfloat[\label{fig:se3_error}]{\includegraphics[width = 2.25in]{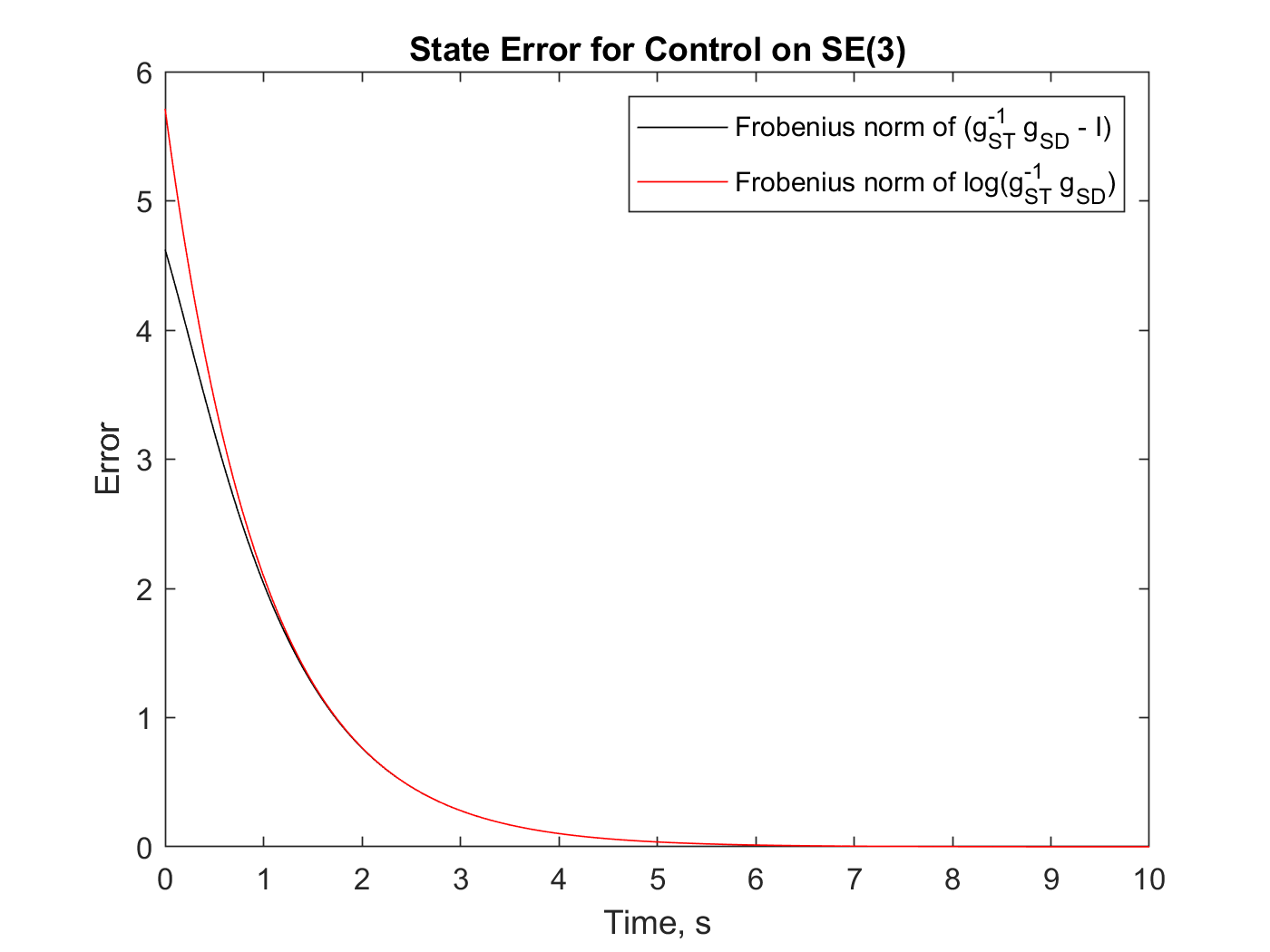}}
\subfloat[\label{fig:su4_error}]{\includegraphics[width = 2.25in]{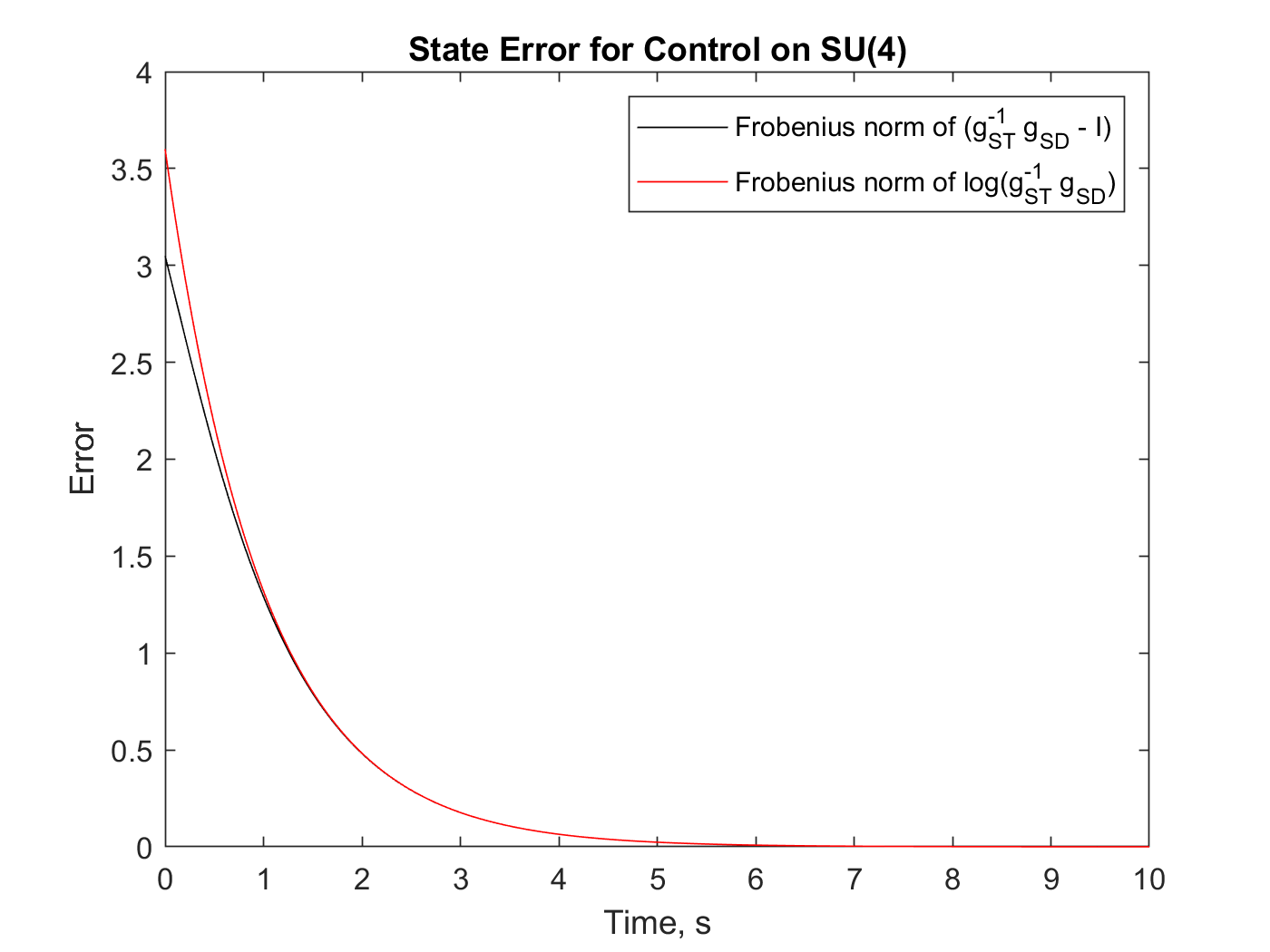}}
\subfloat[\label{fig:gl4_error}]{\includegraphics[width = 2.25in]{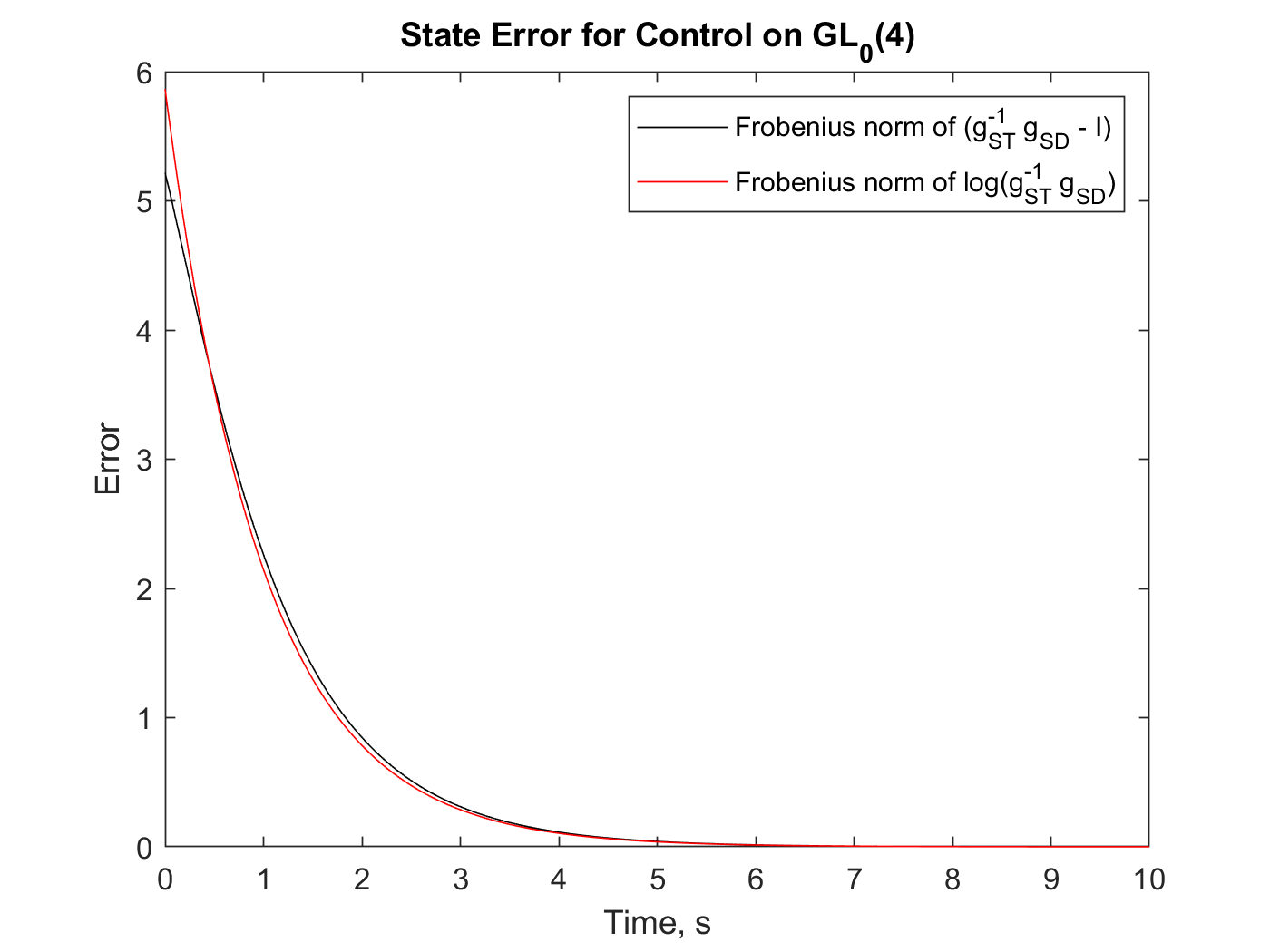}}
\vspace{-0.2cm}
\caption{Visualization of tracking error for simulated systems which evolve on various matrix Lie groups. For all three systems the controller gain was $k = 1$. We display the commonly-used metric $\|g_{TD} - I\|_F$ as well as the norm of the error metric used in the controller $\|\log(g_{TD})\|_F$. (a) Special Euclidean Group $SE(3)$.  (b) Special Unitary Group $SU(4)$. (c) General Linear Group on $\mathbb{R}$ with positive determinant $GL_0(4, \mathbb{R})$.}
\vspace{-0.5cm}
\label{fig:simulation_errors}
\end{figure*}   

\subsection{Proof of Global Exponential Tracking} \label{sec:global_proof}

The primary flaw with the proof in Lemma \ref{lemma:local_stability} is that it relies on a local representation of the matrix logarithm. While algorithms exist that can numerically calculate the matrix algorithm to arbitrary levels of accuracy \cite{higham}, there does not exist a global analytic solution of the matrix log for an arbitrary matrix. Thus, we derive an analytic expression for $\frac{d}{dt} \log(g)$ that does not depend on an analytic expression for $\log(g)$. We begin by considering a small motion in $G$.
\begin{equation}
    g(t + \delta t) = e^{\hat{X}\delta t} g(t) \label{eq:small_motion}
\end{equation}
Since $G$ is a connected matrix Lie group, and $g(t + \delta t)$ and $g(t)$ are elements of $G$, there must exist some matrix $e^{\hat{X} \delta t} \in G$ that fulfills \eqref{eq:small_motion}

\begin{lemma}\label{lemma:small_motion}
Consider \eqref{eq:small_motion}. As $\delta t \rightarrow 0$, $\hat{X}$ approaches the spatial velocity of $g$, $\hat{V}^s$, where $\hat{V}^s = \dot{g}g^{-1}$
\end{lemma}

\begin{proof}
We take the series expansion of the matrix exponential in \eqref{eq:small_motion} to get
\begin{equation}
    g(t + \delta t) = \left(I + \hat{X} \delta t + \mathcal{O}(\delta t^2) \right) g(t)
\end{equation}
And plug it into the limit definition of the derivative
\begin{equation}
    \lim_{\delta t \rightarrow 0} \frac{g(t + \delta t) - g(t)}{\delta t} = \hat{X}g(t) = \dot{g}(t)
\end{equation}
Since $\dot{g}(t) = \hat{X}g(t)$ we have $\hat{X} = \dot{g}(t) g^{-1}(t) = \hat{V}^s$.
\end{proof}

Examining \eqref{eq:small_motion} further, we see that it is of the form
\begin{equation}
    e^{\hat{Z}} = e^{\hat{X}}e^{\hat{Y}}
\end{equation}
Since we're working with exponential coordinates we want an expression for $\hat{Z}$ in terms of $\hat{X}$ and $\hat{Y}$. That expression is the Baker-Campbell-Hausdorff Formula.

\begin{definition}[The Baker-Campbell-Hausdorff Formula] 

Given an equation of the form

\begin{equation}
    e^{\hat{Z}} = e^{\hat{X}}e^{\hat{Y}}
\end{equation}
the Baker-Campbell-Hausdorff (BCH) formula expresses $\hat{Z}$ as a formal series of matrix commutators of $\hat{X}$ and $\hat{Y}$. The first few terms are
\begin{align}
    \hat{Z} = &\hat{X} + \hat{Y} + \frac{1}{2}[\hat{X}, \hat{Y}] \nonumber \\
    &+ \frac{1}{12}[\hat{X}, [\hat{X}, \hat{Y}]] - \frac{1}{12}[\hat{Y}, [\hat{X}, \hat{Y}]] \nonumber \\
    &+ \frac{1}{24}[\hat{Y}, [\hat{X}, [\hat{X}, \hat{Y}]]] + \cdots
\end{align}
where $[A, B] = AB - BA$. This series is infinite and does not necessarily converge. However, when it does converge, it converges to $\log(e^{\hat{X}} e^{\hat{Y}})$. \cite{MugerBCHD, hall2015lie}
\end{definition}

\begin{lemma} \label{lemma:formal_series}
We can represent $\dot{\hat{\xi}} = \frac{d}{dt}\log(g)$ as a formal series of commutators of $\hat{\xi}$ and the spatial velocity $\hat{V}^s = \dot gg^{-1}$.
\end{lemma}
\begin{proof}
We first examine a small motion in $G$. Write
\begin{equation}
    g(t + \delta t) = e^{\hat{\xi}(t + \delta t)} = e^{\hat{X}\delta t} e^{\hat{\xi}(t)}
\end{equation}
Where we pick $\hat X = \frac{1}{\delta t} \log\left(g(t + \delta t) g^{-1}(t)\right)$.
By the BCH formula, we can represent
\begin{equation}
    \hat{\xi}(t + \delta t) = \hat{X}\delta t + \hat{\xi}(t) + \frac{1}{2}[\hat{X}\delta t, \hat{\xi}(t)] + \mathcal{O}(\delta t) + \mathcal{O}(\delta t^2)
\end{equation}
Here $\mathcal{O}(\delta t)$ are first order terms of $\delta t$, and $\mathcal{O}(\delta t^2)$ are terms of second order or higher. We now use the limit definition of the derivative to get
\begin{equation}
    \dot{\hat{\xi}} = \lim_{\delta t \rightarrow 0} \frac{\hat{\xi}(t + \delta t) - \hat{\xi}(t)}{\delta t} = \hat{X} + \frac{1}{2}[\hat{X}, \hat{\xi}(t)] + \frac 1 {\delta t} \mathcal{O}(\delta t)
\end{equation}
By Lemma \ref{lemma:small_motion} this becomes
\begin{equation}
    \dot{\hat{\xi}} = \hat{V}^s + \frac{1}{2}[\hat{V}^s, \hat{\xi}] + \frac{1}{12}[\hat{\xi}, [\hat{\xi}, \hat{V}^s]] + \cdots
\end{equation}
\end{proof}
With this we can now define
\begin{theorem}\label{thm:global_continuous}
A controller of the form
\begin{equation}
    u = k \hat{\xi}_{TD} + g_{TD} \hat{V}^b_{SD} g_{TD}^{-1}
\end{equation}
with $k$ a positive scalar results in local exponential trajectory tracking within the region in which the exponential map $exp : \mathfrak{g} \rightarrow G$ is surjective. For many commonly used matrix Lie groups, such as $SO(n)$, $SE(n)$, $GL(n, \mathbb{C})$ and $SU(n)$, this results in global exponential stability.
\end{theorem}
\begin{proof}
In Lemma \ref{lemma:local_stability} we show this controller results in
\begin{equation}
    \dot{g}_{TD} = -k \hat{\xi}_{TD} g_{TD}
\end{equation}
The spatial velocity is thus
\begin{equation}
    \hat{V}^s_{TD} = \dot{g}_{TD} g_{TD}^{-1} = -k \hat{\xi}_{TD}
\end{equation}
Since $\hat{V}^s_{TD}$ commutes with $\hat{\xi}_{TD}$, all of the commutators in the formal series defined in Lemma \ref{lemma:formal_series} go to zero, leaving us with
\begin{equation}
    \dot{\hat{\xi}}_{TD} = -k \hat{\xi}_{TD}
\end{equation}
Which exponentially goes to zero.
\end{proof}

\subsection{Discussion} \label{sec:proof_remarks}
It should be noted that the control law presented here has a very intuitive geometric interpretation. The control input should be thought of as the sum of two terms: a feedback term $k \hat{\xi}_{TD}$ and a feedforward term $g_{TD} \hat V^b_{SD} g_{TD}^{-1}$. The term $\hat V^b_{SD}$ is simply the body velocity of the reference trajectory, which is the open loop control input to the system. If the system state is adequately tracking the reference trajectory, then it should evolve with this body velocity. The similarity transform $g_{TD} \hat V^b_{SD} g_{TD}^{-1}$ simply re-writes the body velocity to be in the system tool frame $T$. The feedback term $\hat\xi_{TD}$ is the log of $g_{TD}$, which is that velocity which, if executed for 1 second, will bring the instantaneous $T$ frame to the instantaneous $D$ frame. i.e. it is the velocity that will reduce the state error.

\subsection{Proof of Convergence for Discrete Time Control Law} \label{sec:discrete_proof}

In this section we present a proof of convergence for a discrete time version of the control law presented above. As above, let $G$ be a connected matrix Lie group and let $\mathfrak g$ be the associated Lie algebra. Let $g_{ST}(n)$ be the true system state in discrete time, and let $g_{SD}(n)$ be a reference trajectory. Let the discretization time-step be $\Delta t$. In particular, we assume there is some underlying continuous time system $\tilde g_{ST}(t)$ from which $g_{ST}(n)$ is sampled, so that $g_{ST}(n) = \tilde g_{ST}(n \Delta t)$.

\vspace{2mm}
Let $u(n) \in \mathfrak{g}_T$ be the control input. The system evolves according to
\begin{align}
    \label{eq:disc_system}
    g_{ST}(n + 1) =  g_{ST}(n) e^{u(n) \Delta t}
\end{align}
Additionally, let $\hat V^b_{SD}(n)$ be the discrete time body velocity of the reference trajectory. In other words, $\hat V^b_{SD}$ satisfies
\begin{align}
    \label{eq:disc_vel}
    g_{SD}(n + 1) = g_{SD}(n) e^{\hat V^b_{SD}(n) \Delta t}
\end{align}
We again define the configuration error $g_{TD}(n)$
\begin{equation}
    g_{TD}(n) = g_{ST}^{-1}(n) g_{SD}(n)
\end{equation}
and the state error $\hat{\xi}_{TD}(n)$
\begin{equation}
    \hat{\xi}_{TD}(n) = \log(g_{TD}(n))
\end{equation}
As above, we will drop the time-step dependence from our notation where convenient.


\begin{theorem}\label{thm:discrete_proof}
Let $k \in \mathbb{R}^+$ be a positive controller gain, with $k\Delta t < 2$. Then for a small enough discretization step $\Delta t$, a controller of the form
\begin{equation}
    u = k \hat{\xi}_{TD} + g_{TD} \hat{V}^b_{SD} g_{TD}^{-1}
\end{equation}
results in exponential trajectory tracking in the region where the exponential map $exp : \mathfrak{g} \rightarrow G$ is surjective. As in Theorem \ref{thm:global_continuous}, this results in global exponential tracking for many commonly used matrix Lie Groups, including $SE(n)$, $SO(n)$ and $GL(n, \mathbb{C})$.
\end{theorem}

\begin{proof}
Consider the evolution of the configuration error $g_{TD}$.
\begin{align}
        g_{TD}(n+1) &= g_{ST}^{-1}(n+1) g_{SD}(n+1) \nonumber \\
                &= \left(g_{ST}(n)e^{u(n) \Delta t} \right) ^{-1} g_{SD}(n) e^{\hat{V}^b_{SD} \Delta t} \nonumber \\
                &= e^{-u(n) \Delta t} g_{TD}(n) e^{\hat{V}^b_{SD}\Delta t} \label{eq:g_TD(n+1)}
\end{align}
Now examine the control input $u(n)$
\begin{align}
    u &= k \hat{\xi}_{TD} + g_{TD} \hat{V}^b_{SD} g_{TD}^{-1} \nonumber \\
        &= g_{TD}\left(g_{TD}^{-1} k \hat{\xi}_{TD} g_{TD} + \hat{V}^b_{SD} \right) g_{TD}^{-1} \nonumber \\
        &= g_{TD}\left(k \hat{\xi}_{TD} g_{TD}^{-1} g_{TD} + \hat{V}^b_{SD} \right) g_{TD}^{-1} \nonumber \\
        &= g_{TD}\left(k \hat{\xi}_{TD} + \hat{V}^b_{SD} \right) g_{TD}^{-1}
\end{align}
Plugging this into \ref{eq:g_TD(n+1)} we get
\begin{equation}
    g_{TD}(n+1) 
    = g_{TD}(n) e^{-(k \hat{\xi}_{TD}(n) + \hat{V}^b_{TD}(n))\Delta t }e^{\hat{V}^b_{TD}(n) \Delta t } \label{eq:gTD(n+1)subu}
\end{equation}
Examining the final two terms, we have a product of exponentials. We expand the power series' for the matrix exponential, and neglect all the terms with order $\Delta t^2$ or higher.
\begin{gather}
        e^{-(k \hat{\xi}_{TD} + \hat{V}^b_{TD})\Delta t }e^{\hat{V}^b_{TD} \Delta t } = \nonumber \\
        \approx (I - k\hat{\xi}_{TD} \Delta t - \hat{V}^b_{TD} \Delta t)(I + \hat{V}^b_{TD} \Delta t) \nonumber \\
        \approx I - k \hat{\xi}_{TD} \Delta t \nonumber \\
        \approx e^{-k \hat{\xi}_{TD} \Delta t}
\end{gather}
Plugging this back into \ref{eq:gTD(n+1)subu} we get
\begin{gather}
    g_{TD}(n+1) = g_{TD}(n) e^{-k \hat{\xi}_{TD}(n) \Delta t} \nonumber \\
    e^{\hat{\xi}_{TD}(n+1)} = e^{\hat{\xi}_{TD}(n)} e^{-k \hat{\xi}_{TD}(n) \Delta t}
\end{gather}
Since the exponents on the right side commute, we have
\begin{align}
    \hat{\xi}_{TD}(n+1) = (1 - k \Delta t) \cdot \hat{\xi}_{TD}(n) \nonumber \\
    \implies \hat{\xi}_{TD}(n+1)= (1 - k \Delta t)^{n+1} \cdot \hat{\xi}_{TD}(0)
\end{align}
For a given small $\Delta t$, we can choose $k$ so that $|1 - k \Delta t| < 1$. Then $(1 - k\Delta t)^{n+1} \rightarrow 0$ as $n \rightarrow \infty$ and so $\hat \xi_e \rightarrow 0$ exponentially. Since $\hat \xi_e$ goes to zero, the error configuration $g_{TD} = e^{\hat \xi_{TD}}$ goes to the identity matrix, and hence $g_{ST} \rightarrow g_{SD}$ as needed.
\end{proof}

\section{Experimental Results} \label{sec:results}

\subsection{Simulation: Exponential Tracking in $SE(3)$}
We implement our controller in simulation on a system on $SE(3)$. The desired trajectory was a helical trajectory starting at the origin and progressing with constant body velocity $V^b = [0.5, 0.5, 0.3, 0.5, 0.3, 0.7]^T \in se(3)$. We used a control gain $k = 1$. The actual trajectory started at an arbitrary configuration with spectral radius $|g_{TD}(0) - I| > 1$. As shown in figures \ref{fig:se3_traj} and \ref{fig:se3_error}, the system exponentially converges to the desired trajectory.


\subsection{Simulation: Exponential Tracking in $SU(4)$}
We implement our controller in simulation on a system in the Special Unitary Group $SU(4)$. The Lie Algebra for this group $su(4)$ is the set of skew-hermitian matrices in $\mathbb{R}^{4 \times 4}$ \cite{hall2015lie}. The desired trajectory was a randomly determined trajectory with constant body velocity that started at the origin. We used a control gain $k = 1$. The actual trajectory started at an arbitrary configuration with spectral radius $|g_{TD}(0) - I| > 1$. As shown in figure \ref{fig:su4_error}, the system exponentially converges to the desired trajectory.


\subsection{Simulation: Exponential Tracking in $GL_0(4, \mathbb{R})$}

We implement our controller in simulation on a system in $GL_0(4, \mathbb{R})$, the subset of the General Linear Group on $\mathbb{R}$ with positive determinant. The Lie Algebra of $GL_0(4)$ is the set of matrices in $\mathbb{R}^{4x4}$ \cite{hall2015lie}. The desired trajectory started at the origin and was determined by randomly selecting a body velocity for each time step. We used a control gain $k = 1$. The actual trajectory started at an arbitrary configuration with spectral radius $|g_{TD}(0) - I| > 1$. As shown in figure \ref{fig:gl4_error}, the system exponentially converges to the desired trajectory.

\begin{figure}[bh]
    \centering
    \vspace{-0.5cm}
    \includegraphics[ width=\columnwidth]{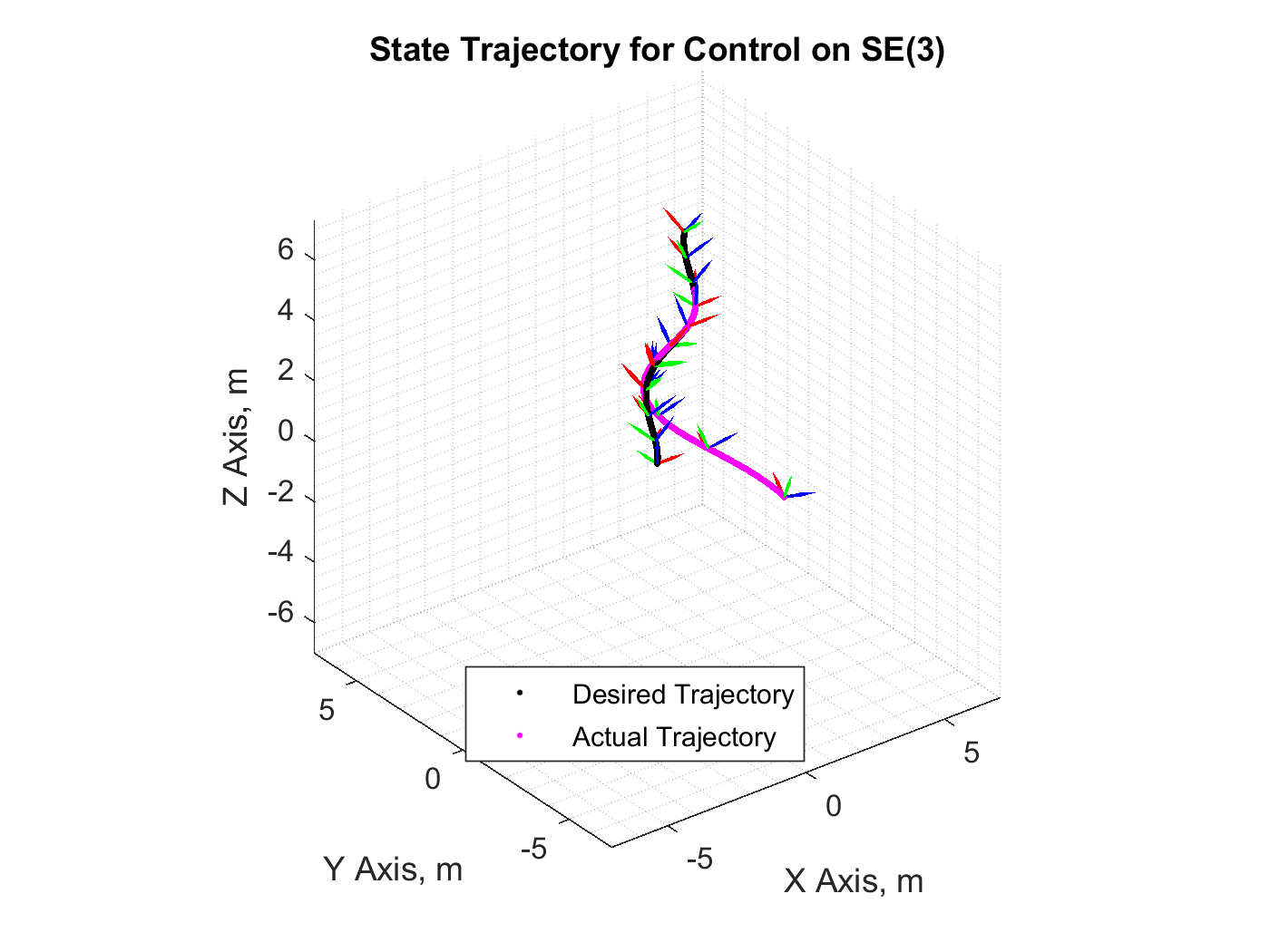}
    \vspace{-0.8cm}
    \caption{Visualization of the desired and actual trajectories for a simulated system which evolves on the Special Euclidean Group $SE(3)$. The desired trajectory starts at the identity element (the origin), while the actual trajectory begins at some offset. The actual trajectory is seen to converge to the desired trajectory.}
    \label{fig:se3_traj}
    \vspace{0cm}
\end{figure}


\subsection{Hardware Experiment: 7-DOF Sawyer Manipulator}
We implement our controller on a 7-DOF Sawyer robot arm. Here, the controller is implemented on $SE(3)$, and is used to do Cartesian velocity control of the end-effector. In particular, we use our control law to get the end-effector of the robot to track a trajectory $g_{SD}(t) \in SE(3)$. We find that this results in an easy to implement controller that requires minimal tuning and results in effective trajectory tracking.

The current state $g_{ST}(t) \in SE(3)$ of the end-effector is computed through the forward kinematics of the manipulator and used to compute the error $g_{TD}(t)$. We use the controller defined in Theorem \ref{thm:discrete_proof} to define the control input $u = \hat{V}_{ST}^b$, which is then transformed into a spatial velocity vector.
\begin{equation}
    V^s_{ST} = \left(g_{ST} \hat{V}_{ST}^b g_{ST}^{-1} \right)^\vee
\end{equation}
This spatial velocity is then converted to a joint-velocity input using the Moore-Penrose pseudoinverse of the manipulator Jacobian, which is then sent as a command to the robot. See \cite{MLS} for more details.
\begin{equation}
    \dot{\theta} = J_{ST}^{s\dagger}(\theta) V_{ST}^s
\end{equation}
We use system specifications from Sawyer's pre-calibrated URDF \cite{URDF} and the OROCOS Kinematics and Dynamics Library \cite{kdl-url} to compute the robot's forward kinematics and Jacobian.

The controller is tested against a helical trajectory where the desired orientation of the arm is kept in a fixed random configuration. The arm is initialized away from the trajectory, in a different position and orientation. In our experiment, we use a gain of $k = 1$. We find that under the presented control law, the end-effector quickly converges to the desired trajectory and then tracks the trajectory with negligible error in steady state. See figure \ref{fig:robot_xyz_error} for plots of desired and true cartesian position during the tracking task. We also plot an error metric $e(t) = ||g_{TD} - I||_F$, which shows that the configuration error converges exponentially to zero (figure \ref{fig:robot_g_error}).



\section{Conclusion}

By developing our control law in the exponential coordinates of the Lie group rather than on the group itself, we're able to sidestep the task of finding easily-differentiable Lyapunov functions, while guaranteeing globally exponential convergence on the most commonly-used Lie groups. We find that this control law is easily-implemented, and highly intuitive, so much so that we had our undergraduate robotics course implement it as a homework assignment.

We feel that the modified BCH formula seen in Lemma \ref{lemma:formal_series} is an interesting start to further research. Since many of the terms in the BCH formula go to zero in a limit, we may be able to find bounds or convergence properties when the control input does not commute with the error. In particular, we feel that this could be a way to try tackling robustness or underactuated control. Further work could also include representing a second order controller using exponential coordinates to compare with \cite{maithripala}.

\begin{figure}[thbp]
\centering
\vspace{-0.3cm}
\subfloat[\label{fig:robot_xyz_error}]{\includegraphics[width = 3.1in]{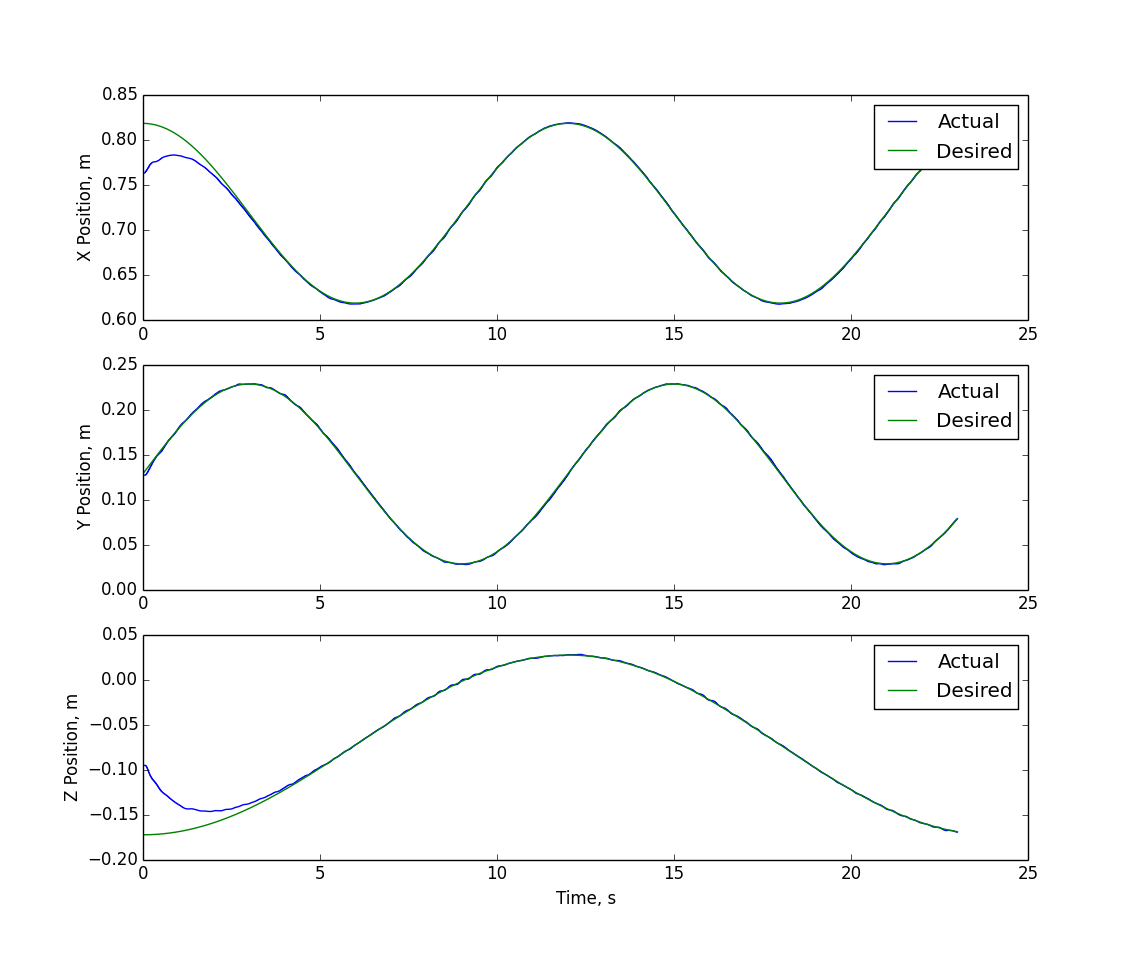}} \\
\vspace{-0.3cm}
\subfloat[\label{fig:robot_g_error}]{\includegraphics[width = 3.1in]{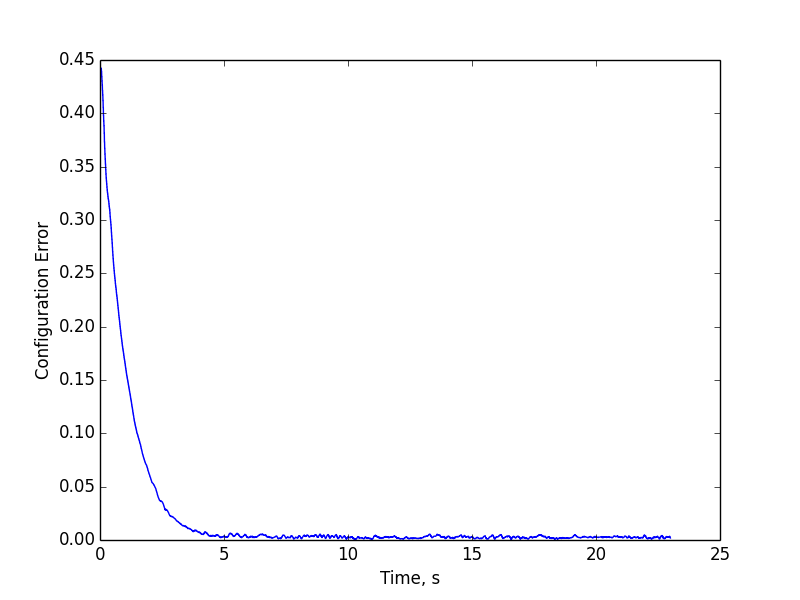}}
\caption{Performance of the Sawyer robot arm while tracking a helical reference trajectory using the presented controller on $SE(3)$ as a cartesian control scheme. (a) Comparison of reference cartesian position and true cartesian position of the Sawyer robot's end effector. (b) Error in the Sawyer robot's end effector configuration $g_{ST}(t) \in SE(3)$ as compared to the reference trajectory $g_{SD}(t)$. The configuration error metric is $e(t) = ||g_{TD} - I||_F$ where $||\cdot||_F$ is the Frobenius matrix norm. The error is seen to go to zero exponentially.}
\label{fig:robot_figures}
\vspace{-0.3cm}
\end{figure}    

\section{Acknowledgments}
We would like to thank Tyler Westenbroek, Joseph Menke, and Chih-Yuan Chiu for discussions and editing help. We would also like to thank the Spring 2020 class of EECS 106B/206B at UC Berkeley for debugging these algorithms in their homework assignments.

\bibliographystyle{./bibliography/IEEEtran}
\bibliography{./bibliography/IEEEabrv,./bibliography/references}

\end{document}